\documentclass{article}
\usepackage{ijcai25}
\usepackage{tikz}

\usepackage{times}
\usepackage{soul}
\usepackage{url}
\usepackage[hidelinks]{hyperref}
\usepackage[utf8]{inputenc}
\usepackage[small]{caption}
\usepackage{graphicx}
\usepackage{amsmath}
\usepackage{amsthm}
\usepackage{booktabs}
\usepackage{algorithm}
\usepackage{algorithmic}
\usepackage[switch]{lineno}
\usepackage{multirow}
\usepackage{paralist}
\usepackage{listings}
\usepackage{amssymb}
\usepackage{mathtools}
\usepackage{enumitem}
\usepackage{bbold}
\usepackage[caption=false]{subfig}
\newtheorem{example}{Example}

\title{Situational-Constrained Sequential Resources Allocation via Reinforcement Learning}

\author{
Libo Zhang$^{1,2}$
\and
Yang Chen$^3$\and
Toru Takisaka$^{1}$\and
Kaiqi Zhao$^2$\and
Weidong Li$^2$\And
Jiamou Liu$^2$\\
\affiliations
$^1$School of Computer Science and Engineering, University of Electronic Science and Technology of China\\
$^2$The University of Auckland \\
$^3$The University of New South Wales\\
\emails
\{lzha797,wli916\}@aucklanduni.ac.nz,
takisaka@uestc.edu.cn,
\{kaiqi.zhao,jiamou.liu\}@auckland.ac.nz,
yang.chen.csphd@gmail.com
}
\usepackage{bibentry}
\newtheorem{definition}{Definition}
\newtheorem{proposition}{Proposition}
\newtheorem*{proposition*}{Proposition}
\newtheorem{problem}{Problem}
\newcommand{\Mmc}[0]{{{\mathcal{M}}}}
\newcommand{\Lmc}[0]{{{\mathcal{L}}}}

\newcommand{\real}{\mathbb{R}}
\DeclareMathOperator*{\argmin}{arg\,min}
\DeclareMathOperator*{\argmax}{arg\,max}
\newcommand{\Amc}[0]{{{\mathcal{A}}}}
\newcommand{\Smc}[0]{{{\mathcal{S}}}}
\begin{document}

\maketitle

\begin{abstract}
    Sequential Resource Allocation with situational constraints presents a significant challenge in real-world applications, where resource demands and priorities are context-dependent. This paper introduces a novel framework, SCRL, to address this problem. We formalize situational constraints as logic implications and develop a new algorithm that dynamically penalizes constraint violations. To handle situational constraints effectively, we propose a probabilistic selection mechanism to overcome limitations of traditional constraint reinforcement learning (CRL) approaches. We evaluate SCRL across two scenarios: medical resource allocation during a pandemic and pesticide distribution in agriculture. Experiments demonstrate that SCRL outperforms existing baselines in satisfying constraints while maintaining high resource efficiency, showcasing its potential for real-world, context-sensitive decision-making tasks.
\end{abstract}

\section{Introduction}\label{sec:Intro}

{\em Sequential resource allocation} (SRA) involves distributing limited resources across locations over time, where an agent allocates resources at a sequence of demand nodes while satisfying upper and lower bound constraints. The objective is to allocate resources efficiently while adhering to these constraints. SRA arises in critical domains such as healthcare, public safety, energy, and agriculture, where dynamic demands and societal priorities play a key role. For example, healthcare resource distribution during pandemics must balance immediate needs with future demand~\cite{malenica2024adaptive}, while pesticide distribution must adapt to regional crop health and sustainability requirements~\cite{qin2021density}.

Beyond efficiency, resource allocation algorithms must consider societal constraints, such as equity~\cite{pu2021fairness}, sustainability~\cite{heffron2014achieving}, and justice~\cite{zhao2020energy}. Moreover, these constraints often depend on context, such as prioritizing equity when regions’ demands conflict. For example, during the COVID-19 pandemic, New Zealand proposed a Traffic Light system \cite{taylor2023impacts} to adjust policies according to the level of emergency. E.g. one-meter distancing measures were only enforced when public medical resources faced high pressure. Similarly, the U.S. clean-energy supply-chain strategy~\cite{igogo2022america} highlighted adaptive systems to address bottlenecks during disruptions. They underscore the need for context-aware allocation strategies.

\begin{figure}[tbp]
 \centering
  \subfloat[Medical Scenario]{
  \includegraphics[width=0.42\columnwidth]{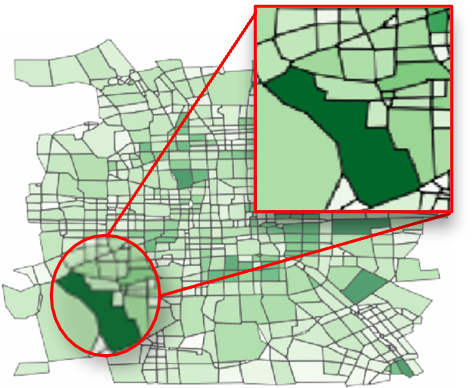}
  \label{fig:BJ-sce}
 }
 \subfloat[Agricultural Scenario]{
  \includegraphics[width=0.37\columnwidth]{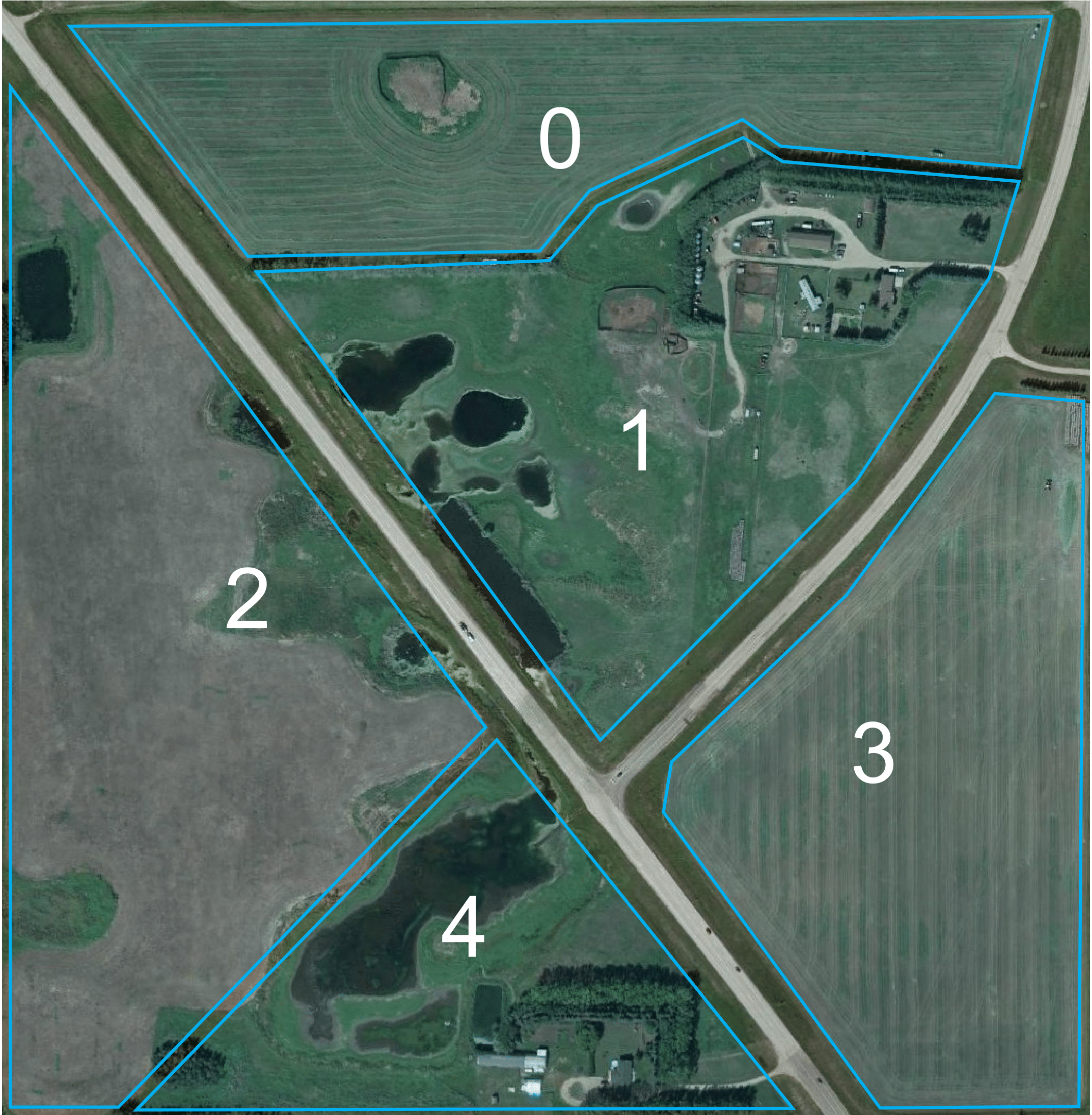}
  \label{fig:agri-sce}
 }
\caption{\textbf{Left:} Simulated medical resource demand in Beijing, where darker colors represent higher demand levels. The experiment focuses on a district in the southwest. \textbf{Right:} Farmland in Saskatchewan, Canada, used for pesticide allocation. Numbers indicate regions with varying pesticide requirements.}\label{fig:sces}
\end{figure}

Traditional solutions to SRA, such as dynamic programming~\cite{lien2014sequential} and multi-armed bandits~\cite{kaufmann2018bayesian}, are effective for small-scale problems with explicit models but struggle with scalability. Reinforcement learning (RL) offers a promising alternative by learning optimal policies through interaction with the environment, without requiring prior system knowledge~\cite{bhatia2019resource}. RL has been successfully applied to diverse SRA tasks, such as pesticide spraying~\cite{qin2021density}, healthcare resource allocation~\cite{li2023deep}, and dynamic electricity distribution~\cite{bahrami2020deep}. Constrained reinforcement learning (CRL) extends standard RL by incorporating constraints into the learning objective, typically through Lagrangian methods or constrained policy updates. Recent advancements, such as density-constrained reinforcement learning (DCRL)~\cite{qin2021density}, have extended RL by incorporating constraints on state distributions. However, existing algorithms rely on static constraints, limiting their ability to adapt to evolving demands and situational requirements. Addressing this gap calls for an advanced RL framework capable of incorporating conditional constraints to enable adaptive decision-making. This leads to the question: {\em How can we design a density-constrained RL framework that ensures situational fairness and adapts to dynamic, context-dependent constraints in resource allocation?} To address this question, one needs to (1) develop a formal framework for SRA under such constraints. (2) Propose a new density-constrained RL algorithm that handles ``if-then'' logic for such constraints. 

In this paper, we initiate the study of situational constraints for SRA tasks. We formulate the problem as a conditional DCRL problem, where the constraints are implications. To address this problem, we propose a new algorithm, Situational-Constrained Reinforcement Learning (SCRL). The algorithm extends the conventional CRL framework by introducing a violation degree-based punitive term function that quantifies the extent of constraint violations and adjusts policy updates accordingly. Unlike previous approaches~\cite{tessler2018reward,Ray2019,qin2021density}, SCRL incorporates an adaptive aggregation mechanism that handles disjunctive constraints by selectively prioritizing 
one constraint within the disjunction.
% the most feasible conditions within the constraint set.
This design allows SCRL to dynamically balance reward optimization with constraint satisfaction in 
% complex, 
context-sensitive environments. To the best of our knowledge, this is the first work to address situational, disjunctive constraints within the CRL paradigm.

We evaluate SCRL in two real-world-inspired scenarios: medical resource allocation during the COVID-19 pandemic in Beijing, China \cite{hao2021hierarchical} and agricultural resource distribution in Saskatchewan, Canada \cite{qin2021density}. In both cases, the constraints are designed to balance equity and adequacy, such as ensuring fairness when resources are insufficient and maintaining sufficient coverage when resources are ample. Experimental results demonstrate that SCRL significantly improves the satisfaction of situational constraints compared to baseline methods and effectively adapts resource distributions to meet context-specific requirements. Additionally, we present a case study to illustrate how SCRL adjusts resource allocation across regions in response to shifting situational demands, further highlighting the algorithm's ability to provide adaptive and equitable decision-making in complex, real-world environments. 

The following is a summary of key contributions:
\begin{itemize}[leftmargin=*]
\item Formulation of sequential resource allocation with situational, disjunctive constraints.
\item Development of the SCRL algorithm with a violation degree-based punitive term for dynamic policy updates.
\item Introduction of an 
% adaptive 
aggregation mechanism to handle disjunctive constraints in context-sensitive environments.
\end{itemize}

\section{Related Work}\label{sec:related}

\paragraph{Sequential Resource Allocation (SRA).}
SRA focuses on distributing resources in systems where demands arrive sequentially, making it distinct from traditional resource allocation due to its dynamic nature and uncertainty. Its relevance spans socially impactful applications, such as allocating medical testing resources during pandemics~\cite{malenica2024adaptive} and optimizing industrial gas deliveries to minimize costs and prevent shortages~\cite{berman2001deliveries}. Ethical considerations, such as equity in resource distribution, have also been explored in government and community planning~\cite{johnson2007community}.

Early approaches to SRA used dynamic programming to optimize costs under uncertainty, including supply chain management for sequential customers~\cite{bassok1995dynamic}. Bayesian methods were later introduced to handle stochastic dynamics, with Bayes-UCB demonstrating asymptotic optimality~\cite{kaufmann2018bayesian}. To address fairness, heuristic algorithms were proposed for equitable and sustainable allocation~\cite{lien2014sequential}.
Reinforcement learning (RL) has recently become a predominant approach for SRA, offering scalable solutions for complex environments. Deep RL has been applied to supply chain management~\cite{peng2019deep} and network slicing~\cite{liu2021clara}, enabling efficient resource allocation under constraints. Resource-constrained RL frameworks have further improved performance over conventional policies~\cite{bhatia2019resource}.

Despite these advancements, existing methods often focus on fixed constraints or single-objective optimization, leaving a gap in addressing situational and context-sensitive constraints, which are critical for real-world applications.

\smallskip

\paragraph{Constrained reinforcement learning (CRL).} 
CRL extends traditional RL by incorporating constraints to ensure policies satisfy predefined requirements while maximizing rewards~\cite{gu2022review,garcia2015comprehensive}. Rooted in Constrained Markov Decision Processes (CMDPs)\cite{altman1993asymptotic}, methods like Reward Constrained Policy Optimization (RCPO)\cite{tessler2018reward} and SAC-Lag~\cite{ha2020learning} use Lagrange multipliers to balance rewards and constraints. Constrained Policy Optimization (CPO)\cite{achiam2017constrained} introduced trust region methods for maintaining feasibility during updates, while Projection-based CPO (PCPO)\cite{yang2020projection} extended it to further avoid infeasible policy during optimization.

Density-based CRL imposes constraints directly on state density functions, offering clear physical interpretations suitable for resource and safety-critical applications~\cite{rantzer2001dual}. Qin et al.\cite{qin2021density} applied this approach to a pesticide spraying scenario by constraining pesticide density, and Zhang et al.\cite{zhang2023learning} extended it to multi-agent settings with ethical constraints. These methods demonstrated the efficacy of density constraints for resource allocation but do not address situational constraints, which require dynamic adaptation across scenarios.

Other RL paradigms are less applicable to SRA. Logic-based RL~\cite{hasanbeig2018logically,hasanbeig2020deep} relies on qualitative specifications, which are unsuitable for quantitative resource allocation and introduce significant computational complexity. 
Fuzzy-logic-based RL, e.g., FQL~\cite{glorennec1997fuzzy}, applies fuzzy rules to represent value functions and actions, and has been used in tasks like robot navigation~\cite{fathinezhad2016supervised} and resource management~\cite{prasath2024combining}. However, it typically yields soft rule satisfaction, which is unsuitable for strict constraints like fairness and safety.
Shielding~\cite{waga2022dynamicshieldingreinforcementlearning,alshiekh2017safereinforcementlearningshielding} focuses on safe exploration, which is unnecessary for SRA with reliable simulations.

Our approach builds on density-based CRL, extending it to handle situational constraints that cannot be addressed by traditional Lagrangian methods. We propose a novel algorithm tailored to this problem, ensuring dynamic and context-sensitive resource allocation.

\section{Problem Formulation}\label{sec:problem}
\subsection{Situational Constraints}
We study resource allocation to demand nodes represented as finite regions within a 2D spatial domain, reflecting applications where resources are distributed geographically. Formally, a supplier agent distributes resources over a bounded region \( D \subseteq \mathbb{R}^2 \) with \( m \in \mathbb{N} \) demand nodes, indexed as \( M = \{1, \ldots, m\} \). Each demand node \( i \in M \) covers a sub-region \( S_i \subseteq D \), and an allocation function \( f \colon M \to \mathbb{R} \) maps each node \( i \) to a non-negative amount of resources \( f(i) \geq 0 \). These sub-regions may overlap. Applications include drones spraying pesticides~\cite{qin2021density}, mobile immunization vehicles distributing vaccines, and policing resource allocation~\cite{maslen2024ethical}, as illustrated in Figure~\ref{fig:sces}.

In many situations, resource allocation must satisfy interval and equity constraints. Interval constraints ensure that each demand node \( i \) receives resources within a specified range \( f(i) \in [a_i, b_i] \), where \( 0 \leq a_i < b_i \leq \infty \)~\cite{qin2021density}. Equity constraints ensure fairness, expressed as \( |f(i) - f(j)| \leq b \), where \( b \geq 0 \)~\cite{lien2014sequential,zhang2023learning}. These constraints address sufficiency and fairness but lack flexibility for context-aware allocation. We therefore introduce situational constraints, which enable conditional relationships between constraints. For example, a situational constraint may state: ``If the resources allocated to region \( A \) exceed a certain threshold, then region \( B \) must receive a minimum amount.'' %This approach allows resource allocation strategies to adapt dynamically to varying contexts.
Let \( \vec{f} = [f(1), \ldots, f(m)] \) represent the allocation vector. 
\begin{definition} 
    An {\em atomic constraint} is of the form $\vec{a}\cdot\vec{f}\leq b$ where vector $\vec{a}\in \real^m$ and $b\in \real$ are parameters.  A {\em situational constraint} is of the form $\varphi_1(\vec{f}) \to \varphi_2(\vec{f}),$ where $\varphi_1(\vec{f})$ and $\varphi_2(\vec{f})$ are atomic. 
\end{definition}
Situational constraints generalize existing formulations. For example, an interval constraint \( f(i) \in [a_i, b_i] \) can be expressed as \( \top \to [-f(i) \leq -a_i] \) and \( \top \to [f(i) \leq b_i] \), where \( \top \) denotes an always-true atomic constraint. Similarly, an equity constraint can be rewritten as a conjunction of two interval constraints.
% \( |f(i) - f(j)| \leq b \) can be rewritten as \( \top \to [f(i) - f(j) \leq b] \) and \( \top \to [f(j) - f(i) \leq b] \).
%
While interval and equity constraints have been studied~\cite{qin2021density,lien2014sequential,zhang2023learning}, they fail to capture context-sensitive requirements. Situational constraints address this gap by capturing conditional demands.
\subsection{SRA with Situational Constraints}\label{sec:mdp-form}
%We formulate the SRA problem as a Markov Decision Process (MDP) to capture its sequential nature, modeling the agent's state, actions, transitions, and rewards, enabling the agent to sequentially decide how to traverse the 2D region and allocate resources efficiently.
Formally, the SRA problem is represented by the MDP \( \mathcal{M} = \langle \mathcal{S}, \mathcal{A}, r, P, \gamma, \eta, L \rangle \), where:
\begin{itemize}[leftmargin=*]
\item \( \mathcal{S} \): The state space, which captures the agent's position in the 2D region and its movement dynamics, such as velocity. %A state \( s \in \mathcal{S} \) provides all necessary information about the agent's current status for decision-making.

\item  \( \mathcal{A} \): The action space, which represents the set of actions available to the agent. % An action \( a \in \mathcal{A} \) moves the agent from one state to another, influencing its trajectory and indirectly determining resource allocation.

\item \( P\colon \mathcal{S} \times \mathcal{A} \to \Delta(\mathcal{S}) \): The transition function, which defines the probability distribution over next states \( s' \) given the current state \( s \) and action \( a \). %The agent’s movement across the region is governed by these transitions.

\item  \( r\colon \mathcal{S} \times \mathcal{A} \to \real \): The reward function, which quantifies the efficiency of resource allocation. % The reward penalizes inefficient resource distribution and rewarding effective allocations.

\item \( \gamma \in (0, 1) \): The discount factor.

\item \( \eta \in \Delta(\mathcal{S})\): The initial state distribution, which specifies the probability distribution over initial states \( s_0 \in \mathcal{S} \).

\item \( L\colon \mathcal{S} \to 2^M \):  The labeling function maps  \( s \in \mathcal{S} \) to the set of demand nodes $L(s)$ receiving resources at that state.
\end{itemize}

To formalize the reward function $r$, we follow \cite{qin2021density} and assume that the agent distributes resources at a constant rate as it moves through the region. This means that the amount of resources $f(i)$ allocated to a demand node \( i \in M \) is captured by the time the agent spends within the corresponding sub-region \( S_i \). This simplification links resource allocation directly to the agent's trajectory, allowing the agent to control resource distribution through its movement across the space. Formally, a {\em trajectory} is a potentially infinite sequence $\tau = (s_0,s_1,s_2,\ldots)$ where each $s_t\in \mathcal{S}$ represents the agent's state at time $t$. The following definition follows the definition in \cite{qin2021density,rantzer2001dual,syed2008apprenticeship,chen2019duality}.

\begin{definition}
Given a trajectory \( \tau \), the \emph{density} of resources allocated to demand node \( i \in M \) is defined as:
\(
\rho^\tau(i) \coloneqq \sum_{t=0}^\infty \gamma^t \cdot \mathbb{1}_{L(s_t)}(i),
\)
where \( \mathbb{1}\) is the indicator function.  
\end{definition}

The density \( \rho^\tau(i) \) quantifies the cumulative, discounted amount of resources allocated to demand node \( i \) by the agent while following the trajectory \( \tau \). The discount factor ensures the density function does not diverge under infinite-horizon settings.
%The discount factor \( \gamma \in (0, 1) \) ensures that earlier allocations are given more weight than later ones, reflecting the time-sensitive nature of resource allocation.
%
 %
A \emph{policy} \( \pi \colon \mathcal{S} \to \Delta(\mathcal{A}) \) defines a probability distribution over the agent's actions at each state. The sequence of states \( \tau = (s_0, s_1, \ldots) \) {\em conforms to} a policy \( \pi \), written \( \tau \sim \pi \), if:
\(
s_{t+1} \sim P(s_t, a_t), \quad a_t \sim \pi(s_t) \text{ \quad  for all } t \geq 0,
\)
where \( P(s_t, a_t) \) is the transition function.

\begin{definition}
Given a policy \( \pi \colon \mathcal{S} \to \Delta(\mathcal{A}) \), the \emph{expected density} of resources allocated to demand node \( i \in M \) is:
\[
\rho^\pi(i) \coloneqq \mathbb{E}_{\tau \sim \pi}[\rho^\tau(i)] = \sum_{t=0}^\infty \gamma^t \Pr(i \in L(s_t) \mid  \pi, \eta),
\]
where \( \Pr(i \in L(s_t) \mid \pi, \eta) \) denotes the probability that demand node \( i \) is receiving resources at time \( t \), given the policy \( \pi \) and initial state distribution \( \eta \).  
\end{definition}

The expected density \( \rho^\pi(i) \) captures the average amount of resources allocated to demand node \( i \) when the agent follows policy \( \pi \). 
Thus any atomic constraint $\varphi(\vec{f})$ of the form $\vec{a}\cdot \vec{f}\leq b$ can be rephrased as the following constraint over the policy $\pi$ of the agent: 
$  a_1 \rho^\pi(1)+\cdots + a_m\rho^\pi(m)\leq b$.
%This reformulation connects the policy \( \pi \), the resulting expected resource allocation \( \rho^\pi(i) \), and the constraints on allocation.
We now formalize our main problem:

\begin{problem}[SRA with situational constraints]~\label{prob:prob1}
\[
\argmax_{\pi}  \sum_{s\in \mathcal{S}} \eta(s) V_\pi(s) \text{ \qquad s.t. \qquad } \Psi(\pi)
\]
where  \( V_\pi(s) \) is the expected cumulative reward under policy \( \pi \), defined as:
\(
   V_\pi(s) = \mathbb{E}_{\tau \sim \pi,s_0=s} \left[ \sum_{t=0}^\infty \gamma^t r(s_t, a_t) \right],
\)
    and \( \Psi(\pi) \) is a conjunction of situational constraints over  $\pi$.  
\end{problem}

\section{Method}\label{sec:alg}
\subsection{Algorithm Overview}\label{sec:AlgOvv}
To solve Problem~\ref{prob:prob1}, we propose a general algorithm framework (Alg.~\ref{alg:Scheme}) designed for constrained reinforcement learning (CRL) tasks. The framework dynamically addresses the dual objectives of maximizing cumulative rewards and satisfying constraints by incorporating a {\em punitive mechanism} through a {\em punitive term function}, denoted as \( \sigma(\Psi, s) \). This term penalizes constraint violations directly within the reward structure, effectively transforming the constrained RL problem into an unconstrained optimization task. The objective then becomes maximizing the cumulative punished reward, allowing the agent to learn constraint satisfaction implicitly while pursuing reward optimization.

The framework operates iteratively through three core steps:
(1) {\bf Trajectory generation:} The current policy \( \pi \) is used to generate a collection of trajectories \( D_\pi \).
(2) {\bf Constraint evaluation and punitive term update:} Using \( D_\pi \), the violation degree \( Vio^\pi(\Psi) \) is computed to evaluate how well the policy satisfies the constraints, and \( \sigma(\Psi, s) \) is updated accordingly.
(3) {\bf Policy optimization:} The policy \( \pi \) is updated by maximizing the cumulative punished reward.
This iterative process continues until convergence, ensuring that both reward maximization and constraint satisfaction objectives are met. This framework generalizes existing CRL methods, including RCPO~\cite{tessler2018reward}, PPO-Lag~\cite{Ray2019}, and DCRL~\cite{qin2021density}, which utilize scalar Lagrangian multipliers as punitive terms. % While effective for conjunctions of atomic constraints, these approaches are limited in handling situational constraints, which introduce conditional relationships between constraints.

\begin{algorithm}[b]
    \caption{Algorithm Scheme with Punitive Term}\label{alg:Scheme}
    \begin{algorithmic}[1]
    \STATE {\bf Input:} An MDP $\Mmc$ with a constraint $\Psi$
    \STATE {\bf Initialize:} An initial policy $\pi$; A punitive term $\sigma$
    \WHILE{Not converged}
        \STATE \textbf{Generate trajectories} $D_\pi\leftarrow \{\tau_1,\tau_2,\cdots\mid \eta,\pi,P\}$
        \STATE \textbf{Evaluate violation degree} $Vio^\pi(\Psi)$ using $D_\pi$\label{line:eval-cons-vio}
        \STATE \textbf{Update punitive term function} $\sigma(\Psi,s)$ %using $Vio^\pi(\Psi)$
        \FOR{each transition $(s,a,r,s') \in \tau_i$ where $\tau_i \in D_\pi$}
            \STATE Apply punitive term on reward $r' \leftarrow r -\sigma(\Psi,s)$
        \ENDFOR
        \STATE \textbf{Update policy} $\pi \leftarrow \argmax_{\pi}\mathbb{E}_{D_\pi}\left[\sum_t \gamma^t r'(s_t,a_t)\right]$ \label{line:update-pi}
    \ENDWHILE
    \STATE Return $\pi$
\end{algorithmic}
\end{algorithm}

We extend the punitive mechanism to handle situational constraints. This involves:
1. Defining the punitive term \( \sigma(\varphi, s) \) for atomic constraints, ensuring it accurately reflects the violation degree.
2. Extending \( \sigma \) to situational constraints, such as \( \sigma(\varphi_1 \to \varphi_2, s) \).

\subsection{Punitive Mechanism: Atomic Constraints}\label{sec:punSing}

For an atomic constraint \( \varphi \) of the form \( \vec{a} \cdot \vec{\rho^\pi} \leq b \), the violation degree is defined as: $Vio^\pi(\varphi) \coloneqq \vec{a} \cdot \vec{\rho^\pi} - b$.
This quantifies the extent to which the constraint \( \varphi \) is violated under the current policy \( \pi \). A positive \( Vio^\pi(\varphi) \) indicates a violation, while a value of zero or less signifies satisfaction of the constraint.

To design a punitive mechanism for atomic constraints, we decompose it into two complementary components: the {\em penalty factor}, which measures the overall severity of a constraint violation, and the {\em weighting factor}, which determines the state-level impact of the violation. These components collectively define the punitive term \( \sigma(\varphi, s) \).

\smallskip

\paragraph{1. Penalty factor:} The penalty factor \( \kappa(\varphi) \) is updated iteratively to reflect the accumulated violation of $\varphi$ over time:
\[
\kappa'(\varphi) \coloneqq \max(0, \kappa(\varphi) + \beta \cdot Vio^\pi(\varphi)),
\]
where  \( \beta \) is the learning rate. This approach aligns with existing CRL methods, such as DCRL~\cite{qin2021density}, where the penalty factor is dynamically adjusted to enforce density constraints. For instance, DCRL enforces an upper bound \( \rho_{max} \) on state density \( \rho^\pi(s') \) by updating \( \kappa \) as $\kappa' = \max(0, \kappa + \beta \cdot (\rho^\pi(s') - \rho_{max}))$.
Our mechanism generalizes this idea to arbitrary atomic constraints.

\smallskip

\paragraph{2. Weighting Factor:} The weighting factor \( w(\varphi, s) \) accounts for the fact that visiting different states may contribute unequally to the violation of \( \varphi \) by setting
\[
w(\varphi, s) \coloneqq
\frac{\sum_{i=1}^m a_i \cdot \mathbb{1}_{L(s)}(i)}{\sum_{i=1}^m |a_i|},
\]
where \( \mathbb{1} \) is the indicator function, and \( L(s) \) identifies the demand nodes affected by state \( s \). Thus the weighting factor \( w(\varphi, s) \) reflects how allocating resources to state \( s \) impacts the violation degree \( Vio^\pi(\varphi) \).

\smallskip

\paragraph{3. Punitive term function:} The punitive term for an atomic constraint \( \varphi \) is defined as:
$\sigma(\varphi, s) = w(\varphi, s) \cdot \kappa(\varphi)$.
The penalized reward is then computed as:
\[
r'(s, a) = r(s, a) - \sigma(\varphi, s).
\]
This formulation ensures that the agent is guided toward satisfying the constraint by dynamically adjusting the reward based on the impact of each state on the violation degree.

\begin{example}\label{exa:weight}
Consider atomic constraint \( \varphi \) that specifies
$\rho^\pi(i)-\rho^\pi(j)\leq 0$, which enforces that the resources allocated to demand node $i$ should not exceed those allocated to demand node $j$. Suppose that the penalty factor $\kappa(\varphi)>0$, i.e., the constraint is violated under the current $\pi$. 
For $s\in \mathcal{S}$, when $i\in L(s)$ and $j\notin L(s)$, by definition, we have $w(\varphi,s)=1/2$. In this case, $\sigma(\varphi,s)=1/2\kappa(\varphi)>0$. This leads to a decrease on the reward, which discourages the agent to allocate resource at $s$. On the other hand, when $i\notin L(s)$ and $j\in L(s)$, $w(\varphi,s)=-1/2$, which means $\sigma(\varphi,s)<0$, leading to an increase on the reward, which encourages the agent to allocate resources at $s$. In any other case, $w(\varphi,s)=0$, which means that any penalty at state $s$ will not affect the satisfaction of $\varphi$. 
\end{example}

With the definition of the punitive term $\sigma(\varphi,s)$ provided above, we can directly instantiate Algorithm~\ref{alg:Scheme} for addressing an SRA problem with an atomic constraint $\varphi$.

\begin{proposition}\label{rem:DCRL-behav}
Consider an SRA problem with an atomic constraint $\varphi$ of the form $\vec{a}\cdot \vec{\rho}^\pi\leq b$. With a sufficiently small learning rate $\beta$, the algorithm framework incorporating the punitive term $\sigma(\varphi,s)$ converges to a feasible solution. 
\end{proposition}
\begin{proof}
The proof builds on principles from canonical Lagrangian-based CRL frameworks, such as \cite{tessler2018reward}. Appendix.~\ref{Sec:App-Alg} contains proof details.
\end{proof}

\subsection{Punitive Mechanism: Situational Constraints}\label{sec:punDis}
Situational constraints \( \psi \coloneqq \varphi_1 \to \varphi_2 \) can be reformulated as disjunctive constraints \( \neg \varphi_1 \lor \varphi_2 \). This requires handling disjunctions of atomic constraints \( \varphi_1 \vee \varphi_2 \), which pose unique challenges. Traditional CRL methods work well for conjunctive constraints as they enable gradient-based optimization within connected feasible regions. However, disjunctive constraints create disconnected feasible regions that make gradient-based methods ineffective.

Conventional approaches usually reformulate disjunctive constraints into mixed-integer linear programming (MILP)~\cite{trespalacios2015improved,kronqvist2021between} and solve them using techniques such as branch-and-bound~\cite{turkay1996disjunctive}. While effective for problems with explicit system models (e.g., linear programming), these methods struggle with the complexity and dynamic nature of realistic tasks like the SRA problem, where constraints are context-dependent and the environment evolves sequentially. Moreover, branch-and-bound approaches scale poorly in large or high-dimensional problems due to their exhaustive exploration of disjuncts. Importantly, recent machine learning methods have tackled disjunctive constraints by interpreting them as a \(\min\)-operator over loss functions or as unions of feasible sets~\cite{ren2020query2box,huang2022line,li2019augmenting,nandwani2019primal}. Building on these ideas, our approach extends the use of \(\min\)-operators to RL by designing a punitive mechanism for disjunctive constraints.

%For a disjunctive constraint $\psi$ of the form $\bigvee_{j\in J} \varphi_j$ where each $\varphi_j$ is atomic, our mechanism dynamically adjusts the punitive term based on the agent's policy and trajectory. 
To define the punitive term $\sigma(\psi,s)$ on a state $s$, we adopt the principle that prioritizes the ``least-violated'' disjunct in the constraint, i.e., $\varphi_j$ where $j=\argmin_j \{\sigma(\varphi_1, s), \sigma(\varphi_2, s), \ldots, \sigma(\varphi_j, s)\}$. The $\min$ operator aligns with the logical semantics of disjunctions and has been applied in prior works~\cite{ren2020query2box,huang2022line}. This design encourages the policy to satisfy the most attainable disjunct in a disjunction. An illustrative example is provided in Figure~\ref{fig:algtree} in the Appendix.

While it is intuitive to select the least-violated atomic constraint during the optimization, greedily applying the $\min$-operator may suffer from sub-optimality in cases where it focuses on an infeasible disjunct that appears easier to satisfy. Specifically, consider a disjunction \( \psi \coloneqq \varphi_1 \lor \varphi_2 \), where \( \varphi_1 \) is infeasible for the policy set \( \Pi \), meaning \( \forall \pi \in \Pi, \, Vio^\pi(\varphi_1) > 0 \), while \( \varphi_2 \) is feasible. In this scenario, for a given policy \( \pi' \in \Pi \) at state \( s \), if the punitive term for the infeasible disjunct \( \varphi_1 \), \( \sigma(\varphi_1, s) \), is smaller than that for the feasible disjunct \( \varphi_2 \), \( \sigma(\varphi_2, s) \), the algorithm would encourage the policy to prioritize \( \varphi_1 \). 
Consequently, the policy is misguided to focus on an unattainable constraint. 
Figure~\ref{fig:weiDft} in Appendix.~\ref{Sec:App-Alg}. illustrates this issue.
% , where infeasible constraints are represented by sets with no intersection with \( \Pi \), and the distance from \( \pi' \) to each set corresponds to the punitive term. 

%\noindent {\bf Probabilistic Mechanism.}  
To address this issue, we propose a probabilistic mechanism for selecting disjuncts within a disjunctive constraint. For a disjunction \( \psi = \bigvee_{j \in [J]} \varphi_j \), where \( \varphi_j \) are atomic constraints, probabilistic mechanism defines a random variable \( \Phi \) over the set of atomic constraints. The punitive term \( \sigma(\psi, s) \) is then defined as \( \sigma(\Phi, s) \), where \( \Phi \) follows a probability distribution that assigns a probability \( p_j \) to each \( \varphi_j \). These probabilities are defined as \( p_j \coloneqq \frac{\kappa(\varphi_j)^{-1}}{\sum_{j \in J} \kappa(\varphi_j)^{-1}} \), where \( \kappa(\varphi_j) \) is the penalty factor associated with \( \varphi_j \). This formulation ensures that constraints closer to satisfaction are prioritized, while constraints with larger penalty factors are still occasionally explored due to their nonzero probabilities. 
%By introducing a randomized selection mechanism, our algorithm mitigates the deterministic behavior. 
%It allows the agent to navigate fragmented feasible regions more effectively and reduces the risk of being trapped in infeasible constraints. 
Further details on the implementation are provided in Algorithm~\ref{alg:sig_hat} in Appendix~\ref{Sec:App-Alg}, and its performance is evaluated through an ablation study.

Finally, for a conjunction of $I$ situational constraints, $\Psi \coloneqq \land_{i \in I} \varphi_i$, we define $\sigma(\Psi,s)=\sum_{i \in I}\sigma(\psi_i,s)$.

\subsection{SCRL Algorithm}\label{sec:AlgProc}
Alg.\ref{alg:SCRL} presents the Situational-Constrained Reinforcement Learning (SCRL) algorithm, an instance of the general framework (Alg.\ref{alg:Scheme}) tailored for Problem~\ref{prob:prob1}. SCRL incorporates the punitive mechanism defined above to handle situational constraints.
The algorithm initializes the penalty factor \( \kappa(\varphi) = 0 \) for each atomic constraint \( \varphi \), with a learning rate \( \beta \) to encourage exploration during the early stages of training. Trajectory data is used to empirically estimate the state density \( \rho^\pi(s) \), leveraging either discrete state counts or kernel-based methods for continuous spaces~\cite{qin2021density,chen2017tutorial}, ensuring computational efficiency for large-scale problems.
The algorithm iteratively alternates between generating trajectories, updating penalty factors based on constraint violations, applying punitive terms to the rewards, and optimizing the policy. This iterative process ensures both constraint satisfaction and reward maximization.

\begin{algorithm}[tb]
%\small
% \renewcommand{\thealgorithm}
   \caption{Situational-Constrained RL}\label{alg:SCRL}
\begin{algorithmic}[1]
   \STATE {\bf Input:} An MDP 
$(\Smc, \Amc, P, r, \eta, \gamma, L)$, situational constraints $\Psi := \bigwedge_{i \in I} \bigvee_{j \in \{1,2\}}\varphi_j$
   \STATE {\bf Initialisation:} Let $\pi$ be a random policy, $\kappa(\varphi) = 0$ be penalty factor for each $\varphi$, $\beta$ be learning rate for $\kappa$
   \REPEAT \label{line:mainLoopStarts}
   \STATE Generate trajectories $D_\pi=\{\tau_1,\tau_2,\cdots\mid \eta,\pi,P\}$ \label{line:generateExperience}
   \STATE Empirically compute density $\rho^\pi$ according to $D_\pi$ \label{line:computeDensityEstimation}
   \FORALL{atomic constraint $\varphi$}
        \STATE Compute the violation degree $Vio^\pi(\varphi)$ for $\varphi$ \label{line:computeViolation}
        \STATE Update penalty factor as: \\ $\kappa(\varphi) \leftarrow \max(0,\kappa(\varphi)+\beta  Vio^\pi(\varphi))$ \label{line:updateSigma}
    \ENDFOR 
   \FOR{each $\tau_i \in D_\pi$, each transition $(s,a,r,s') \in \tau_i$} \label{line:modifyRewardStarts}
        \STATE Calculate $\sigma(\Psi,s)$. \label{line:computeHat}
        \STATE Apply punitive term on reward $r' \leftarrow r -\sigma(\Psi,s)$\label{line:modifyReward}
   \ENDFOR\label{line:modifyRewardEnds}
   \STATE Solve $\pi$ that maximizes the expected punished return based on $D_\pi$\label{line:updatePol}
   \UNTIL Convergence\label{line:mainLoopEnds}
   \STATE {\bfseries Output:} A policy $\pi$, with density values $\rho^\pi$.
\end{algorithmic}
\end{algorithm}

\section{Experiment}\label{sec:Exp}
We aim to validate SCRL algorithm's performance through empirical evaluations over two real-world scenarios.

\subsection{Experiment Scenarios and Tasks}

\paragraph{Medical Resource Allocation.} This scenario models the allocation of medical resources in Beijing during the COVID-19 pandemic using a simulation from~\cite{hao2021hierarchical}. The city is divided into modules, grouped into five sub-regions based on demand levels (Figure~\ref{fig:med-regions}). The challenge is to prioritize high-demand regions during resource shortages while maintaining fairness, reflecting real-world public health requirements for dynamic, context-sensitive allocation policies. 

\paragraph{Agricultural Spraying Drone.}~\cite{qin2021density} This scenario involves pesticide allocation in farmland in Saskatchewan, Canada, divided into five sub-regions based on crop types (Figure~\ref{fig:agri-regions}). The agent must optimize pesticide usage by responding to pest outbreaks while avoiding overuse, balancing sufficiency and fairness across regions. This mirrors real-world agricultural challenges with economic and environmental implications.

The detailed experiment setting of these two scenarios can be found in the Appendix.~\ref{sec:apx-exp}.
For both scenarios, the agent’s goal is to maximize resource allocation efficiency while satisfying constraints. Three tasks of increasing complexity evaluate the agent's performance: 1. {\bf Situational Task:} A single situational constraint requires, e.g., ``If resources allocated to certain regions exceed a threshold, others must receive a minimum allocation.'' This tests the agent's ability to adapt dynamically to conditional requirements. 2. {\bf Priority Task:} Involves equity constraints (equal resource allocation across specific regions) and adequacy constraints (minimum resources for specific regions). The situational requirement states, ``If adequacy cannot be met, ensure equity.'' This evaluates the agent's ability to prioritize fairness under resource limitations. 3. {\bf Joint Task:} Combines adequacy and equity constraints simultaneously without prioritization, requiring the agent to balance potentially conflicting requirements. In some cases, satisfying both constraints may be infeasible. 
Details on scenarios and tasks are provided in Appendix~\ref{sec:apx-exp}.

\begin{figure}[!tbp]
 \centering
  \subfloat[Med-Regions]{
  \includegraphics[width=0.37\columnwidth]{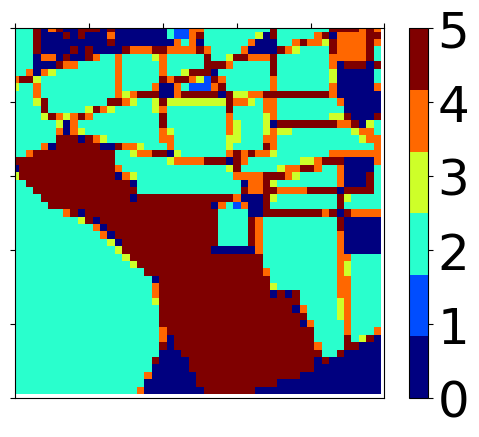}
  \label{fig:med-regions}
 }
 \subfloat[Agri-Regions]{
  \includegraphics[width=0.37\columnwidth]{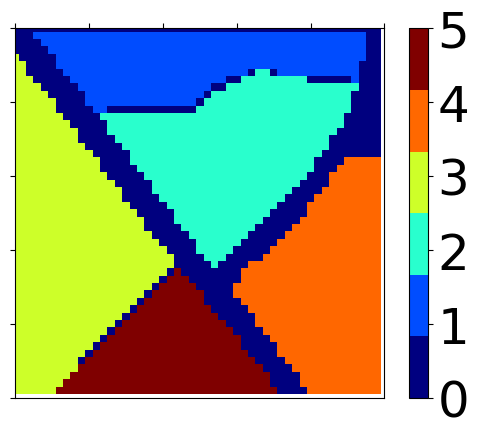}
  \label{fig:agri-regions}
 }
\caption{
The sub-regions in two scenarios. For both scenarios, the map is divided into 50$\times$50 grids, and 5 regions (regions with label 0 are ignored in our settings).
}\label{fig:regions}
\end{figure}

\subsection{Baselines}
We compare our approach against four baseline methods:
\begin{enumerate}[leftmargin=*]
    \item Deep Deterministic Policy Gradient {\bf (DDPG)}~\cite{gu2017deep} serves as an unconstrained RL baseline, optimizing solely for reward without considering constraints. 
    % It provides a theoretical upper bound on reward but is expected to fail in satisfying task-specific constraints.

    \item Reward Constrained Policy Optimization {\bf (RCPO)}~\cite{tessler2018reward} is adapted to our setting by defining cost functions over state-action pairs instead of using density-based constraints. RCPO cannot natively support density-based situational constraints, requiring approximations for implementation.

    \item Conservative Augmented Lagrangian {\bf (CAL)}~\cite{wu2024off} is a recent primal-dual CRL method. We adapt CAL to our setting in the same manner as RCPO. 
    
    \item Density Constrained Reinforcement Learning {\bf (DCRL)}~\cite{qin2021density} is included as a baseline but lacks native support for situational constraints. To adapt, we decompose each situational constraint \( \varphi_1 \to \varphi_2 \) into: {\bf (DCRL1):} The premise \( \neg \varphi_1 \), treated as an interval constraint \( \rho^\pi(s) \in [a, b] \). {\bf (DCRL2):} The conclusion \( \varphi_2 \), addressed independently as an interval constraint.
\end{enumerate}

\subsection{Performance Metrics}

The primary evaluation metric, {\bf constraint violation (Cons.Vio.)}, measures violation of constraints. A lower Cons.Vio. indicates better compliance with these critical constraints and is preferred, as we prioritize the safety and fairness in real-world applications. The density function in experiments takes an undiscounted sum due to the finite-horizon setting. The secondary metric, {\bf reward}, assesses resource efficiency, with higher rewards indicating less resource allocation amount. A negative reward is assigned when one unit of resources is allocated.

\begin{table}[tb]
		\scriptsize
  % \small
			\centering 
\begin{tabular}{lllrr}
\toprule
\textbf{Scenario}                & \textbf{Task}                         & \textbf{Alg.} & \multicolumn{1}{c}{\textbf{Cons.Vio}} & \multicolumn{1}{c}{\textbf{Reward}} \\ \hline
\multirow{17}{*}{\textbf{Med.}}  & \multirow{6}{*}{\textbf{Situational}} & DDPG          & 4.84$\pm$0.28                         & \textbf{-2.09$\pm$0.01}             \\
                                 &                                       & RCPO          & 5.07$\pm$0.33                         & \textbf{-2.09$\pm$0.01}             \\
                                 &                                       & CAL           & 16.84$\pm$13.13                       & -4.05$\pm$2.16                      \\
                                 &                                       & DCRL1         & \textbf{0.0$\pm$0.0}                  & -14.65$\pm$1.02                     \\
                                 &                                       & DCRL2         & 10.15$\pm$9.17                        & -11.65$\pm$13.94                    \\
                                 &                                       & SCRL          & \textbf{0.0$\pm$0.0}                  & -13.12$\pm$3.86                     \\ \cline{2-5} 
                                 & \multirow{6}{*}{\textbf{Priority}}    & DDPG          & 30.16$\pm$10.61                       & -2.74$\pm$2.03                      \\
                                 &                                       & RCPO          & 33.36$\pm$0.75                        & \textbf{-2.1$\pm$0.01}              \\
                                 &                                       & CAL           & \multicolumn{1}{l}{13.31$\pm$7.37}    & -4.67$\pm$1.90                      \\
                                 &                                       & DCRL1         & 32.15$\pm$4.44                        & -3.38$\pm$4.06                      \\
                                 &                                       & DCRL2         & 7.95$\pm$14.14                        & -29.77$\pm$19.75                    \\
                                 &                                       & SCRL          & \textbf{0.0$\pm$0.0}                  & -18.52$\pm$13.17                    \\ \cline{2-5} 
                                 & \multirow{5}{*}{\textbf{Joint}}       & DDPG          & 54.04$\pm$0.23                        & \textbf{-2.09$\pm$0.01}             \\
                                 &                                       & RCPO          & 54.04$\pm$0.22                        & -2.1$\pm$0.01                       \\
                                 &                                       & CAL           & \textbf{36.18$\pm$0.4}                & -5.02$\pm$3.02                      \\
                                 &                                       & DCRL          & 41.03$\pm$7.9                         & -29.97$\pm$18.16                    \\
                                 &                                       & SCRL          & 37.17$\pm$8.85                        & -9.62$\pm$4.17                      \\ \hline
\multirow{17}{*}{\textbf{Agri.}} & \multirow{6}{*}{\textbf{Situational}} & DDPG          & 4.85$\pm$0.17                         & \textbf{-2.1$\pm$0.0}               \\
                                 &                                       & RCPO          & 11.49$\pm$1.55                        & -3.2$\pm$0.3                        \\
                                 &                                       & CAL           & 6.71$\pm$3.95                         & -3.79$\pm$3.34                      \\
                                 &                                       & DCRL1         & 3.52$\pm$4.76                         & -13.69$\pm$3.15                     \\
                                 &                                       & DCRL2         & 4.69$\pm$4.93                         & -19.15$\pm$14.76                    \\
                                 &                                       & SCRL          & \textbf{0.01$\pm$0.03}                & -45.37$\pm$9.76                     \\ \cline{2-5} 
                                 & \multirow{6}{*}{\textbf{Priority}}    & DDPG          & 10.37$\pm$0.3                         & \textbf{-2.1$\pm$0.02}              \\
                                 &                                       & RCPO          & 1.26$\pm$0.67                         & -3.22$\pm$0.18                      \\
                                 &                                       & CAL           & 7.64$\pm$2.35                         & -3.74$\pm$3.36                      \\
                                 &                                       & DCRL1         & \textbf{0.0$\pm$0.0}                  & -31.29$\pm$16.37                    \\
                                 &                                       & DCRL2         & 3.88$\pm$4.9                          & -15.22$\pm$19.15                    \\
                                 &                                       & SCRL          & \textbf{0.0$\pm$0.0}                  & -9.46$\pm$1.20                      \\ \cline{2-5} 
                                 & \multirow{5}{*}{\textbf{Joint}}       & DDPG          & 19.54$\pm$0.15                        & \textbf{-2.1$\pm$0.01}              \\
                                 &                                       & RCPO          & 13.22$\pm$0.92                        & -3.04$\pm$0.21                      \\
                                 &                                       & CAL           & 18.61$\pm$1.23                        & -2.39$\pm$0.23                      \\
                                 &                                       & DCRL          & 20.3$\pm$10.66                        & -24$\pm$11.82                       \\
                                 &                                       & SCRL          & \textbf{4.31$\pm$2.38}                & -24.53$\pm$5.07                     \\ \bottomrule
\end{tabular}
\caption{Result on two scenarios, each with three tasks. Mean and Std are collected from 10 independent runs.}
			\label{tab:main-tab}
\end{table}
\subsection{Experiment Result}
\noindent 
The result in Table~\ref{tab:main-tab} shows that DDPG attains high rewards at the cost of severe constraint violations. Both RCPO and CAL rely on a surrogate cost function, so they suffer high violations (with the lone exception of CAL on the Med. joint task). Few DCRL instances satisfy constraints due to the separate implementation. However, DCRL lacks scalability against situational constraints. SCRL consistently offers near-zero cost on priority and situational tasks. For joint task, SCRL also offers low-level costs, demonstrating its advantage in satisfying context-sensitive requirements. Besides, even when a few DCRL instances satisfy constraints, they perform worse on rewards (e.g., -31 vs. -9). In contrast, SCRL successfully guides the agent to higher-reward feasible solutions.

% While DCRL occasionally satisfies constraints, it performs worse on rewards compared to SCRL (e.g., -14 vs. -13, -31 vs. -9) due to its reliance on interval-based constraints, which limits the feasible policy sets. SCRL’s design effectively balances constraint satisfaction and resource efficiency, guiding the agent toward feasible solutions with higher rewards.
\begin{table}[!t]
\centering  \scriptsize
\begin{tabular}{cllrr}
\toprule
\multicolumn{1}{l}{\textbf{Scenario}} & \multicolumn{1}{l}{\textbf{Task}} & \textbf{Alg.} & \textbf{Equity} & \textbf{Adequacy} \\ \midrule
\multirow{4}{*}{\textbf{Med}}         & \multirow{2}{*}{\textbf{Situational}}  & \textbf{DDPG} & \textbf{7.0}            & 125.1          \\
                                      &                                        & \textbf{SCRL} & 24.8           & \textbf{0.0}             \\ \cline{2-5} 
                                      & \multirow{2}{*}{\textbf{Joint}}        & \textbf{DDPG} & \textbf{7.0}            & 50.0           \\
                                      &                                        & \textbf{SCRL} & 9.6            & \textbf{23.1}           \\ \midrule
\multirow{4}{*}{\textbf{Agri}}        & \multirow{2}{*}{\textbf{Situational}}  & \textbf{DDPG} & \textbf{6.9}            & 22.9           \\
                                      &                                        & \textbf{SCRL} & 15.7           & \textbf{0.0}             \\ \cline{2-5} 
                                      & \multirow{2}{*}{\textbf{Joint}}        & \textbf{DDPG} & 6.9            & 22.9           \\
                                      &                                        & \textbf{SCRL} & \textbf{3.8}            & \textbf{0.0}             \\ \bottomrule
\end{tabular}
\caption{Violation of different constraints. Violation of two constraints (equity and adequacy) in different tasks are shown. In Med. scenario, the adequacy requirement for two tasks are different, thus DDPG has different performances.}
\label{tab:cs}
\end{table}
\subsection{Case Study: Priority and Joint Tasks}\label{sec:BehavAna}
\noindent We investigate SCRL agent's behavior by comparing it with DDPG agent's behavior, as an unconstrained baseline.
\begin{figure}[!t]
 \centering
\subfloat[Agri-Joint]{
  \includegraphics[width=0.35\columnwidth]{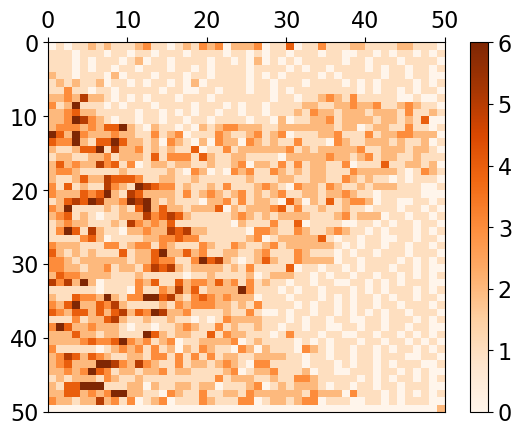}
  \label{fig:agri-equity}
 }
 \subfloat[Agri-Situational]{
  \includegraphics[width=0.35\columnwidth]{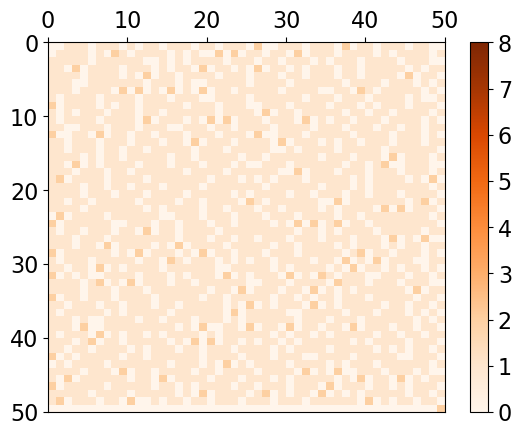}
  \label{fig:agri-safety}
 }
\caption{
SCRL's resource allocation as heatmap. Higher temprerature indicates more resources allocated. In joint task, agent tries to satisfy the equity constraint so allocating resources to different regions in different amounts; In situational task, agent prioritizes adequacy constraint so allocating each region with sufficient resources.
}\label{fig:heatmaps}
\end{figure}

\paragraph{Constraint Violation.}
As shown in Table~\ref{tab:cs}, SCRL effectively prioritizes adequacy constraints in priority tasks, satisfying them at the expense of higher equity violations. In joint tasks, SCRL balances both adequacy and equity constraints, minimizing total violation degrees. This demonstrates SCRL’s adaptability to diverse constraint structures.

\paragraph{Resources Allocation.}
Figure~\ref{fig:heatmaps} visualizes SCRL’s allocation strategies on Agri. scenario, with warmer colors indicating higher allocations. In joint tasks, SCRL balances equity constraints, reflecting region size variations. In priority tasks, it emphasizes adequacy, ensuring critical regions meet minimum demands. These results demonstrate SCRL’s adaptability to dynamically meet task-specific constraints.
A similar analysis for Med. scenario is given in  Appendix.~\ref{sec:apx-exp}.

\paragraph{Resource Efficiency.}
In the Agri. scenario, DDPG minimizes resource usage, completing its trajectory in 800 time steps, achieving higher rewards at the cost of constraint violations. In contrast, SCRL meets adequacy constraints, requiring at least 900 time steps to satisfy minimum demands across regions, leading to lower rewards. A similar trade-off is observed in the Med. scenario. This underscores the inherent trade-off between reward maximization and constraint satisfaction, as CRL algorithms prioritize constraint compliance over unconstrained efficiency.
\subsection{Case Study: More disjunctions}
% \usepackage{multirow}
\begin{table}[]
\centering
\scriptsize
\begin{tabular}{clrr}
\hline
\textbf{Scenario}                                                                  & \textbf{Algorithm} & \multicolumn{1}{l}{\textbf{Cons.Vio}} & \multicolumn{1}{l}{\textbf{Reward}} \\ \hline
\multirow{8}{*}{\textbf{\begin{tabular}[c]{@{}c@{}}Multi\\ Disjunct\end{tabular}}} & \textbf{DDPG}      & 8.6                                   & \textbf{-2.1}                       \\
                                                                                   & \textbf{RCPO}      & 10.8                                  & -4.2                                \\
                                                                                   & \textbf{CAL}       & 10                                    & -2.3                                \\
                                                                                   & \textbf{DCRL1}     & 9.3                                   & \textbf{-2.1}                       \\
                                                                                   & \textbf{DCRL2}     & 9.4                                   & -14.2                               \\
                                                                                   & \textbf{DCRL3}     & \textbf{0}                            & -16.7                               \\
                                                                                   & \textbf{DCRL4}     & \textbf{0}                            & -48.3                               \\
                                                                                   & \textbf{SCRL}      & \textbf{0}                            & -11.8                               \\ \hline
\end{tabular}
\caption{Case study result on multi-disjunction task, Agri. scenario. The constraint involves four disjunctions.}
\label{tab:multi-disj}
\end{table}
Case study in Table.~\ref{tab:multi-disj} shows algorithms performance under multiple disjunctions. The result is consistent with earlier analysis: RCPO and CAL fail due to the surrogate constraint; DCRL occasionally satisfies the constraint with a loss of reward (DCRL 3 and 4); SCRL still shows its advantage in capturing situational constraints.

\subsection{Ablation Study: Probabilistic Mechanism}\label{sec:abl}
\begin{table}[t]
\scriptsize
\centering
\begin{tabular}{clrr}
\toprule
\textbf{Scenario}               & \multicolumn{1}{l}{\textbf{Alg.}} & \multicolumn{1}{c}{\textbf{Cons.Vio}} & \multicolumn{1}{c}{\textbf{Reward}} \\ \hline
\multirow{2}{*}{\textbf{Med.}}  & \textbf{SCRL-min}                 & 0.7$\pm$2.3                              & \textbf{-16.0$\pm$12.0}                \\
                                & \textbf{SCRL}                & \textbf{0.0$\pm$0.0}                     & -18.5$\pm$13.2                         \\ \hline
\multirow{2}{*}{\textbf{Agri.}} & \textbf{SCRL-min}                 & 0.3$\pm$0.8                              & -13.6$\pm$8.2                           \\
                                & \textbf{SCRL}                & \textbf{0.0$\pm$0.0}                     & \textbf{-9.5$\pm$1.2}                  \\ \bottomrule
\end{tabular}
\caption{Ablation study on probabilistic factor. We show the results on situational tasks in two scenarios. {\bf SCRL-min} replaces probabilistic factor with min-operator, as defined in Sec.~\ref{sec:punDis}.}
\label{tab:ablation}
\end{table}
We compare the proposed probabilistic mechanism (SCRL) with a variant using the $\min$ operator (SCRL-min). As shown in Table~\ref{tab:ablation}, SCRL-min performs similarly in reward but incurs slightly higher constraint violations. This aligns with our earlier claim that the $\min$ operator may mislead the agent to satisfy an infeasible constraint with a lower violation.

\section{Conclusion and Limitations}\label{sec:Con}
This paper tackles sequential resource allocation (SRA) under situational constraints, formulating the problem as an MDP with density functions to quantify regional resource allocation. We propose the SCRL algorithm, which trains agents under context-sensitive constraints. 

{\bf Limitations.} This study focuses on a single resource allocator and single-resource type. Extending the framework to multi-agent systems and multi-resource scenarios offers promising directions for future work.
\bibliographystyle{named}
\bibliography{ijcai25}
\clearpage
\section{Appendix}
\subsection{Further Algorithm Explanations}\label{Sec:App-Alg}

\noindent{\bf Proposition and Proof.} Here we put the proof for the Proposition.~\ref{rem:DCRL-behav}.
\begin{proposition*}
Consider an SRA problem with an atomic constraint $\varphi$ of the form $\vec{a}\cdot \vec{\rho}^\pi\leq b$. With a sufficiently small learning rate $\beta$, the algorithm framework incorporating the punitive term $\sigma(\varphi,s)$ converges. 
\end{proposition*}
\begin{proof}

Recall that a canonical CRL problem incorporates a cost function $c(s) \in \real$ with a threshold $\alpha$, and is formulated as:
\begin{align}
    & \arg\max_{\pi} \sum_{s \in \mathcal{S}} \eta(s)V_\pi(s)  \nonumber \\
    & \text{ s.t. } \mathbb{E}[\sum_t \gamma^t c(s_t) \mid \eta, \pi] \leq \alpha.
\end{align}

For simplicity, let $J^r_\pi$ and $J^c_\pi$ represent the expected cumulative reward and cost under policy $\pi$, respectively. A typical Lagrangian-based CRL method uses a Lagrange function $\Lmc(\sigma, \pi) = J^r_\pi - \sigma \times J^c_\pi$, where $\pi$ is updated by maximizing $\Lmc(\sigma, \pi)$, and the multiplier $\sigma$ is adjusted incrementally with learning rate $\beta$ based on the violation degree $J^c_\pi - \alpha$. 
An CRL algorithm usually optimize a policy with a punished reward $r'(s,a)=r(s,a)-\sigma \times c(s)$, reflecting the objective $J^r_\pi - \sigma \times J^c_\pi$.
Under mild assumptions, the algorithm converges almost surely to a feasible solution when the constraint is satisfied~\cite{tessler2018reward}.
We will prove that our algorithm with punitive term function $\sigma(\varphi,s)$ converges to a feasible local optimum of SRA problem with atomic constraint $\varphi$ based on the convergence above.

Consider an atomic constraint $\varphi \colon= \sum_{m \in M} a_m \times \rho^\pi(m) \leq b$, where $c_\varphi(s) = \sum_{m \in M} a_m \cdot \mathbb{1}(m \in L(s))$ serves as the cost function with threshold $b$. Then:
\begin{align}
    J^{c_\varphi}_\pi &= \sum_t \gamma^t \sum_{s \in \Smc} \Pr(s_t = s) c_\varphi(s) \nonumber \\
    &= \sum_t \gamma^t \sum_{m \in M} a_m \Pr(m \in L(s_t)). \nonumber
\end{align}
This shows $J^{c_\varphi}_\pi - b = Vio^\pi(\varphi)$, indicating that $J^{c_\varphi}_\pi \leq b$ if and only if $\pi$ satisfies $\varphi$. Therefore, the canonical CRL problem with $J^{c_\varphi}_\pi \leq b$ shares the same feasible policy set and optimal solutions as the SRA problem with constraint $\varphi$.

In our algorithm framework, the punitive term function $\sigma(\varphi, s)$ modifies the reward as $r' = r - \sigma(\varphi, s) = r - \kappa(\varphi) \times w(\varphi, s)$. By replacing $w(\varphi, s)$ with $c_\varphi(s) / \sum_{m \in M} |a_m|$, the reward becomes:
\[
r' = r - \kappa(\varphi) \times c_\varphi(s) / \sum_{m \in M} |a_m|.
\]
Letting $d = \sum_{m \in M} |a_m|$ and defining $\lambda(\varphi) := \kappa(\varphi) / d$, the reward is rewritten as $r' = r - \lambda(\varphi) \times c_\varphi(s)$. With $\kappa(\varphi)$ updated as $\kappa^o(\varphi) + \beta \times Vio^\pi(\varphi)$, this implies $\lambda(\varphi) = \lambda^o(\varphi) + (\beta / d) \times Vio^\pi(\varphi)$, where $\kappa^o(\varphi)$ and $\lambda^o(\varphi)$ represents the $\kappa(\varphi)$ and $\lambda(\varphi)$ in the previous iteration.

Thus, $\lambda(\varphi)$ can be interpreted as a Lagrangian multiplier with learning rate $\beta / d$, and the policy is updated to maximize the cumulative punished reward $\mathbb{E}[\sum_t \gamma^t r'(s_t) \mid \eta, P]$. Under standard assumptions, the algorithm converges to a feasible solution of the canonical CRL problem with constraint $J^{c_\varphi}_\pi \leq b$.

Finally, due to the equivalence of feasible policy sets between the SRA problem with atomic constraints and the reformulated CRL problem with $c_\varphi$, our algorithm with the punitive term $\sigma(\varphi, s)$ converges almost surely to a feasible local optimum. If applied with Actor-Critic-based RL algorithms, e.g., RCPO~\cite{tessler2018reward}, convergence to a feasible local optimum is also guaranteed under mild assumptions.

\end{proof}

\noindent{\bf Punitive Term via min Operator.}
Here we show how a disjunction can be represented by a min-operator by visualization in Figure.~\ref{fig:algtree}. We take the values of $\sigma(\psi)$ as example, and the idea simply follows canonical arithmetical representation.
\usetikzlibrary{shapes,arrows,positioning}

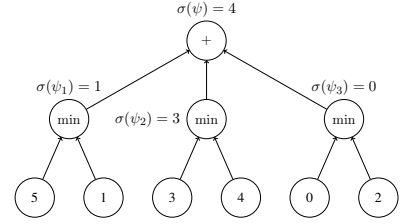
\begin{figure}
    \centering
\resizebox{0.6\columnwidth}{!}{%
\begin{tikzpicture}[
  level/.style={sibling distance=35mm/#1},
  edge from parent/.style={<-,draw},
  every node/.style={circle,draw,minimum size=10mm},
  level 1/.style={level distance=20mm},
  level 2/.style={level distance=20mm},
  level 3/.style={level distance=25mm}
]

\node {$+$}
  child {
    node {$\min$}
    child {
      node {5}
    }
    child {
      node {1}
    }
  }
  child {
    node {$\min$}
    child {
      node {3}
    }
    child {
      node {4}
    }
  }
  child {
    node {$\min$}
    child {
      node {0}
    }
    child {
      node {2}
    }
  };

\node[align=center,draw=none] at(0,0.8) {\large $\sigma(\psi)=4$};
\node[align=center,draw=none] at(-3.5,-1.2) {\large $\sigma(\psi_1)=1$};
\node[align=center,draw=none] at(3.5,-1.2) {\large $\sigma(\psi_3)=0$};
\node[align=center,draw=none] at(-1.5,-2) {\large $\sigma(\psi_2)=3$};
\end{tikzpicture}
}
\caption{Example of representing $\psi$. The number is the value of $\sigma(\psi_{i,j})$, and boolean operators are represented by arithmetical operators.}
    \label{fig:algtree}
    
\end{figure}

\noindent{\bf  Punitive Term for Disjunction.} The following algortihm explains how the probabilistic mechanism works in defining punitive term for a disjunctive constraint.

\begin{algorithm}
    \caption{Calculating punitive term at state $s$ for $\psi := \lor_{j \in J} \varphi_j$}\label{alg:sig_hat}
\begin{algorithmic}[1]
    \STATE {\bf Input:} $\psi := \lor_{j \in [J]} \varphi_j, s, \kappa$
    \IF{$\kappa(\varphi_j)>0, \forall j$ }
        \STATE $u=\sum_{j \in J}\kappa(\varphi_j)^{-1}$
        \FOR{$j \in J$}
            % \IF{$w(\psi_j,s) \neq 0$}
            \STATE $p_j \leftarrow \kappa(\varphi_j)^{-1} / u,$
            % \ELSE
                % \STATE $p_j \leftarrow 0$
            % \ENDIF
        \ENDFOR
        \STATE return $\kappa(\varphi_j)\times w(\varphi_j,s)$ with probability $p_j$
    \ELSE
        \STATE return 0
    \ENDIF
\end{algorithmic}
\end{algorithm}

\noindent{\bf Potential Sub-optimality of $\min$ Operator.} The Figure.~\ref{fig:weiDft} below explains why $\min$ operator may lead to sub-optimal solution. In Figure.~\ref{fig:weiDft}, infeasible constraints are represented by sets with no intersection with \( \Pi \), and the distance from \( \pi' \) to each set corresponds to the punitive term. 
\begin{figure}
    \centering
    \includegraphics[width=0.75\linewidth]{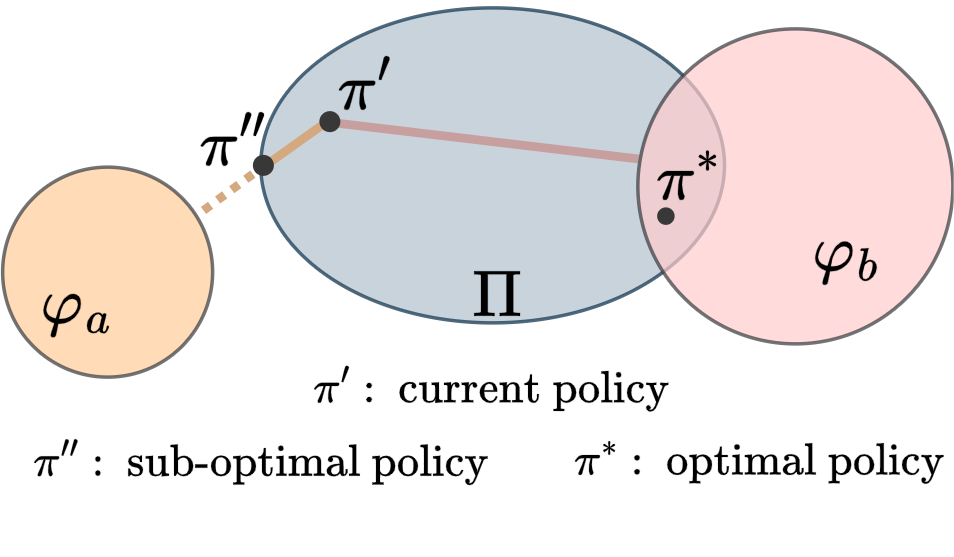}
    \caption{Illustration of sub-optimality in taking $\min$ operator. The grey ellipse represents entire policy space $\Pi$, with the orange (left) and pink (right) regions denoting the feasible sets of $\varphi_a$ and $\varphi_b$, respectively. Since $\varphi_a$ is infeasible (no intersection with $\Pi$), the algorithm misguides the policy $\pi'$ toward satisfying $\varphi_a$, resulting in the sub-optimal policy $\pi''$, despite $\varphi_b$ being feasible.}
    \label{fig:weiDft}
\end{figure}

\begin{table*}[!htbp]
      \small
                \centering 
\begin{tabular}{ll}
\hline
\multicolumn{1}{c}{\textbf{Tasks}} & \multicolumn{1}{c}{\textbf{Task Specification}}                                                                                                                                                        \\ \hline
\textbf{Agri.-Situational}         & $\neg (\rho^\pi(2) \leq 300) \to \rho^\pi(3) \geq 800$                                                                                                                                             \\
\textbf{Agri.-Priority}            & $\neg\biggl(\bigl(\rho^\pi(1)\geq 300\bigr) \land \bigl(\rho^\pi(3)\geq 300\bigr) \land \bigl(\rho^\pi(4)\geq 300\bigr)\biggr) \rightarrow \bigl(\rho^\pi(1)+\rho^\pi(4)= \rho^\pi(3)\bigr) $                      \\
\textbf{Agri.-Joint}               & $\bigl(\rho^\pi(1)+\rho^\pi(4)= \rho^\pi(3)\bigr) \land \bigl(\rho^\pi(1)\geq 300\bigr) \land \bigl(\rho^\pi(3)\geq 300\bigr) \land \bigl(\rho^\pi(4)\geq 300\bigr) $                                              \\ \hline
\textbf{Med.-Situational}          & $\neg (\rho^\pi(1)+\rho^\pi(2) \leq 800) \to (\rho^\pi(3)+\rho^\pi(4)+\rho^\pi(5) \geq 1200)$                                                                                                \\
\textbf{Med.-Priority}             & $\neg\biggl(\bigl(\rho^\pi(2)\geq 1000\bigr) \land \bigl(\rho^\pi(4)\geq 200\bigr) \land \bigl(\rho^\pi(5)\geq 800\bigr)\biggr) \rightarrow \bigl(\rho^\pi(1)+\rho^\pi(2)= \rho^\pi(3)+\rho^\pi(4)+\rho^\pi(5)\bigr) $ \\
\textbf{Med.-Joint}                & $\bigl(\rho^\pi(1)+\rho^\pi(2)= \rho^\pi(3)+\rho^\pi(4)+\rho^\pi(5)\bigr) \land \bigl(\rho^\pi(2)\geq 300\bigr) \land \bigl(\rho^\pi(3)\geq 200\bigr) \land \bigl(\rho^\pi(4)\geq 200\bigr) $                          \\ \hline
\end{tabular}\caption{Constraints in each task. Number indicates the units of resources allocation.}
                \label{tab:specs}
        \end{table*}
\subsection{Experiment Settings}~\label{sec:apx-exp}
In our experiment, each scenario is divided into a 50 $\times$ 50 grids, with five sub-regions over it.
The agent navigates through the entire map, i.e. district in medical resource allocation and the farmland in the agricultural resource allocation.
A state indicates the coordinate of the agent and the current local dynamics, e.g. current moving velocity.
The agent's action controls the local dynamic, thus influences the transition and resulting density function, referring to the resources allocation amount in our problem formulation.
Under our MDP formulation, the reward specifies the resource efficiency.
Therefore, the reward is negatively related with the speed, as faster the agent moves, less resources amount is allocated.
The reward is given negatively related with the speed, specifically, at each step $t$, $r_t = -0.1 / (1 + v_t)$, where $v_t$ is the velocity.
This reward setting aligns with the agricultural pesticide allocation from~\cite{qin2021density}.

\subsection{Task Specifications}\label{sec:det}

The constraints are explained, where one unit of resources allocated to a region $S_i$ refers to 1 in the density value of a region $S_i$. 
In the Medical Resource Allocation scenario, we have:
\begin{itemize}
    \item {\bf Situational Task.} If low-demand regions (\( S_1, S_2 \)) receive more than 800 units, ensure high-demand regions (\( S_3, S_4, S_5 \)) receive at least 1200 units.
    \item {\bf Priority Task.} If \( S_2, S_4, S_5 \) cannot receive sufficient resources (1000, 200, and 800 units, respectively), enforce equal allocation between low-demand (\( S_1, S_2 \)) and high-demand (\( S_3, S_4, S_5 \)) regions.
    \item {\bf Joint Task.} Ensure \( S_2, S_3, S_4 \) receive sufficient resources (300, 200, and 200 units, respectively), while maintaining equal allocation between low-demand (\( S_1, S_2 \)) and high-demand (\( S_3, S_4, S_5 \)) regions.
\end{itemize}
In the Agricultural Resource Allocation scenario, we have:
\begin{itemize}
    \item {\bf Situational Task.} If \( S_2 \) receives more than 300 units of pesticide, ensure \( S_3 \) receives at least 800 units.
    \item {\bf Priority Task.} If \( S_1, S_3, S_4 \) cannot receive sufficient pesticides (300 units each), ensure equal allocation between two crops: one in \( S_4 \), and another across \( S_1 \) and \( S_3 \).
    \item {\bf Joint Task.} Ensure \( S_1, S_3, S_4 \) receive 300 units each, while maintaining equal allocation between two crops: one in \( S_4 \), and another across \( S_1 \) and \( S_3 \).
\end{itemize}

The table below shows the situational constraints we use in our experiment tasks. Numbers are proportional to the units of resources allocated to each region, thus can be empirically evaluated by the density function.

\subsection{Further Case Study}~\label{sec:apx-cs}
Here we present a similar heat map for the case study on Medical resources allocation scenario in the main body. The result is consistent with the case study in the main body: When the agent attempts at the equity constraint, the resources allocation among regions varies, and when agent prioritizes the adequacy constraint, the temperature differences on the heat map get smoother.
\begin{figure}[!tbhp]
 \centering
 \subfloat[Med-Joint]{
  \includegraphics[width=0.4\columnwidth]{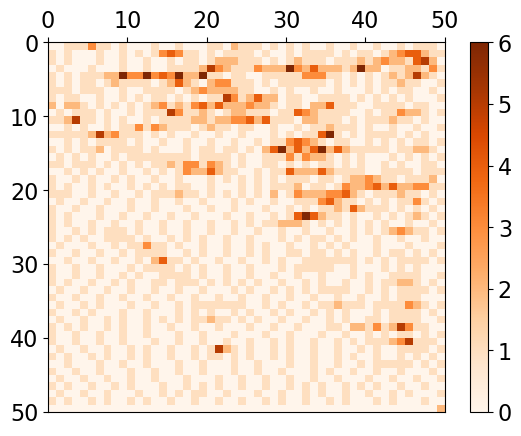}
  \label{fig:med-equity}
 }
 \subfloat[Med-Situational]{
  \includegraphics[width=0.4\columnwidth]{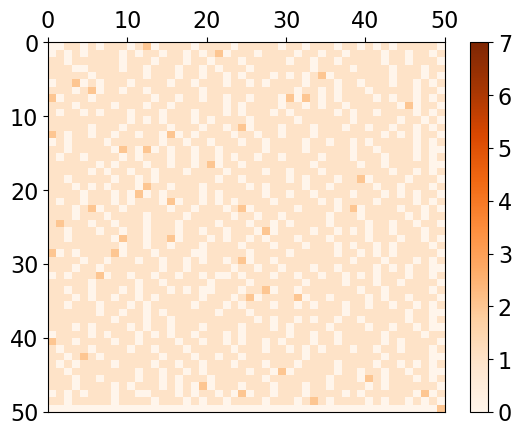}
  \label{fig:med-safety}
 }
\caption{
SCRL's resources allocation as heatmap. In the heatmaps, higher value indicates more resources allocated. In joint task, agent tries to satisfy the equity constraint so allocate resources to different regions with different amount; In situational task, agent priorities safety constraint so allocate each region with sufficient resources.
}\label{fig:heatmaps-ful}
\end{figure}

\end{document}